\documentclass{article}
\usepackage{ijcai19}

\usepackage{times}
\usepackage{soul}
\usepackage{url}
\usepackage[hidelinks]{hyperref}
\usepackage[utf8]{inputenc}
\usepackage[small,skip=0pt]{caption}
\usepackage{booktabs}       % professional-quality tables
\usepackage{amsfonts}       % blackboard math symbols
\usepackage{nicefrac}       % compact symbols for 1/2, etc.
\usepackage{microtype}      % microtypography
\usepackage{acronym}

\usepackage{amsmath,amsthm,amssymb}

\usepackage{color}
\definecolor{Blue}{rgb}{0.9,0.3,0.3}
\usepackage{cancel}

\newcommand{\squishlist}{
   \begin{list}{$\bullet$}
    { \setlength{\itemsep}{0pt}      \setlength{\parsep}{3pt}
      \setlength{\topsep}{3pt}       \setlength{\partopsep}{0pt}
      \setlength{\leftmargin}{1.5em} \setlength{\labelwidth}{1em}
      \setlength{\labelsep}{0.5em} } }

\newcommand{\squishlisttwo}{
   \begin{list}{$\bullet$}
    { \setlength{\itemsep}{0pt}    \setlength{\parsep}{0pt}
      \setlength{\topsep}{0pt}     \setlength{\partopsep}{0pt}
      \setlength{\leftmargin}{2em} \setlength{\labelwidth}{1.5em}
      \setlength{\labelsep}{0.5em} } }

\newcommand{\squishend}{
    \end{list}  }

\newcommand{\myvec}[1]{\mathbf{#1}}
\newcommand{\myvecsym}[1]{\boldsymbol{#1}}

\newcommand{\valpha}{\myvecsym{\alpha}}

\newcommand{\vc}{\myvec{c}}

\newcommand{\vr}{\myvec{r}}

\newcommand{\vA}{\myvec{A}}

\newcommand{\vN}{\myvec{N}}

\newcommand{\vR}{\myvec{R}}

\newcommand{\vT}{\myvec{T}}
\newcommand{\vU}{\myvec{U}}

\newcommand{\vX}{\myvec{X}}
\newcommand{\expect}[1]{\mathbb{E}\left[ {#1} \right]}

\newcommand{\be}{\begin{equation}}
\newcommand{\ee}{\end{equation}}
\newcommand{\bea}{\begin{eqnarray}}
\newcommand{\eea}{\end{eqnarray}}
\newcommand{\beaa}{\begin{eqnarray*}}
\newcommand{\eeaa}{\end{eqnarray*}}

\DeclareMathAlphabet{\mathpzc}{OT1}{pzc}{m}{n}
\newcommand{\qqquad}{\qquad\qquad}

\usepackage{tabu}
\usepackage{paralist}
\usepackage{color}
\usepackage[ruled]{algorithm2e}
\usepackage[noend]{algorithmic}
\usepackage{amsbsy}
\usepackage{amsmath}
\usepackage{bm}
\usepackage{dsfont}
\usepackage[all]{xy}

\usepackage{graphicx}
\usepackage{epstopdf}
\usepackage{epsfig}
\graphicspath{{./figures/}}

\usepackage[capitalize,noabbrev]{cleveref}

\DeclareMathOperator{\argmin}{arg\,min}
\DeclareMathOperator{\argmax}{arg\,max}
\newtheorem{assumption}{Assumption}
\newtheorem{mydefinition}{Definition}

\newtheorem{theorem}[mydefinition]{Theorem}

\newcommand{\cDsi}{\mathcal{D}^*_e}
\newcommand{\cDci}{\bar{\mathcal{D}}_e}
\newcommand{\cDst}{\mathcal{D}^*_t}
\newcommand{\cDct}{\bar{\mathcal{D}}_t}
\newcommand{\cD}{\mathcal{D}}

\newcommand{\cR}{\mathcal{R}}

\newcommand{\abs}[1]{\left|#1\right|}

\newcommand{\set}[1]{\left\{#1\right\}}

\mathchardef\mhyphen="2D

\newcommand{\cascadeducb}{{\sf CascadeDUCB}}
\newcommand{\cascadeswucb}{{\sf CascadeSWUCB}}
\newcommand{\batchrank}{{\sf BatchRank}}
\newcommand{\bubblerank}{{\sf BubbleRank}}
\newcommand{\cascadeklucb}{{\sf CascadeKL\mhyphen UCB}}
\newcommand{\cascadeucb}{{\sf CascadeUCB1}}
\newcommand{\rankedexpthree}{\sf RankedEXP3}
\newcommand{\ducb}{{\sf DUCB}}
\newcommand{\swucb}{{\sf SWUCB}}
\newcommand{\expthree}{{\sf Exp3}}

\newcommand{\ucb}{{\sf UCB}}
\newcommand{\barvalphat}{\bar{\valpha}_t}
\newcommand{\cT}{\mathcal{T}}

\usepackage{enumitem}
\setlist[description]{leftmargin=\parindent,labelindent=\parindent}

\acrodef{IR}{information retrieval}
\acrodef{OLTR}{online learning to rank}
\acrodef{LTR}{learning to rank}
\acrodef{CM}{cascade model}
\acrodef{CB}{cascading bandits}
\acrodef{MAB}{multi-armed bandit}

\if0
\let\OLDthebibliography\thebibliography
\renewcommand\thebibliography[1]{
  \OLDthebibliography{#1}
  \setlength{\parskip}{2.7pt}
  \setlength{\itemsep}{2pt plus 0.3ex}
}
\fi

\begin{document}

\title{Cascading Non-Stationary Bandits: Online Learning to Rank in the Non-Stationary Cascade Model}

\author{
	Chang Li
	\And
	Maarten de Rijke
	\affiliations
	University of Amsterdam 
	\emails
	\{c.li, derijke\}@uva.nl
}

\maketitle

\begin{abstract}
Non-stationarity appears in many online applications such as web search and advertising. 
In this paper, we study the online learning to rank problem in a \emph{non-stationary} environment where user preferences change abruptly at an unknown moment in time. 
We consider the problem of identifying the $K$ most attractive items and propose \emph{cascading non-stationary bandits}, an online learning variant of the \emph{cascading model}, where a user browses a ranked list from top to bottom and clicks on the first attractive item. 
We propose two algorithms for solving this non-stationary problem: $\cascadeducb$ and $\cascadeswucb$. 
We analyze their performance and derive gap-dependent upper bounds on the $n$-step regret of these algorithms. 
We also establish a lower bound on the regret for cascading non-stationary bandits and show that both algorithms match the lower bound up to a logarithmic factor.
Finally, we evaluate their  performance on a real-world web search click dataset. 
\end{abstract}

%!TEX root = Paper.tex

\section{Introduction}
\label{sec:introduction}
Learning to rank \acused{LTR}\acs{LTR}~\cite{liu2009learning} is a combination of machine learning and information retrieval. 
It is a core problem in many applications, such as web search and recommendation~\cite{liu2009learning,batchrank}. 
The goal of \ac{LTR} is to rank items, e.g., documents, and show the top $K$ items to a user. 
Traditional \ac{LTR} algorithms are supervised, offline algorithms; they learn rankers from human annotated data~\cite{qin2010letor} and/or users' historical interactions~\cite{joachims2002optimizing}. 
Every day billions of users interact with modern search engines and leave a trail of interactions. 
It is feasible and important to design online algorithms that directly learn from such user clicks to help improve users' online experience. 
Indeed, recent studies show that even well-trained production rankers can be optimized by using users' online interactions, such as clicks~\cite{clicklambda}. 

Generally, interaction data is noisy~\cite{joachims2002optimizing}, which gives rise to the well-known exploration vs.\ exploitation dilemma. 
Multi-armed bandit~(\acs{MAB})\acused{MAB}~\cite{auer02finitetime} algorithms have been designed to balance exploration and exploitation. 
Based on \acp{MAB}, many online \ac{LTR} algorithms have been published~\cite{radlinski08learning,kveton15cascading,katariya16dcm,pbmbandit,batchrank,kveton2018bubblerank}. 
These algorithms address the exploration vs. exploitation dilemma in an elegant way and aim to maximize user satisfaction in a stationary environment where users do not change their preferences over time. 
Moreover, they often come with regret bounds. 

Despite the success of the algorithms mentioned above in the stationary case, they may have linear regret in a non-stationary environment where users may change their preferences abruptly at any unknown moment in time. 
Non-stationarity widely exists in real-world application domains, such as search engines and recommender systems~\cite{yu2009piecewise,pereira2018analyzing,wu2018learning,jagerman-when-2019}.  
Particularly, we consider \emph{abruptly changing environments} where user preferences remain constant in certain time periods, named \emph{epochs}, but change occurs abruptly at unknown moments called \emph{breakpoints}. 
The abrupt changes in user preferences give rise to a new challenge of balancing ``remembering" and ``forgetting"~\cite{besbes2014stochastic}: the more past observations an algorithm retains the higher the risk of making a biased estimator, while the fewer observations retained the higher stochastic error it has on the estimates of the user preferences. 

In this paper, we propose \emph{cascading non-stationary bandits}, an online variant of the \ac{CM}~\cite{craswell08experimental} with the goal of identifying the $K$ most attractive items in a non-stationary environment. 
\ac{CM} is a widely-used model of user click behavior~\cite{chuklin2015click,batchrank}. 
In \ac{CM}, a user browses the ranked list from top to bottom and clicks the first attractive item.
The items ranked above the first clicked item are browsed but not attractive since they are not clicked. 
The items ranked below the first clicked item are not browsed since the user stops browsing the ranked list after a click. 
Although \ac{CM} is a simple model, it effectively explains user behavior~\cite{kveton15cascading}. 

Our key technical contributions in this paper are: 
\begin{inparaenum}[(1)]
	\item We formalize a  non-stationary \ac{OLTR} problem as cascading non-stationary bandits. 
	\item  We propose two algorithms, $\cascadeducb$ and $\cascadeswucb$, for solving it. 
	They are  motivated by \emph{discounted UCB} ($\ducb$) and \emph{sliding window UCB} ($\swucb$), respectively~\cite{garivier-2011-nonstationary}. 
	$\cascadeducb$ balances ``remembering" and ``forgetting" by using a discounting factor of past observations, and $\cascadeswucb$ balances the two by using statistics inside a fixed-size sliding window. 
	\item We derive gap-dependent upper bounds on the regret of the proposed algorithms. 
	\item We derive a lower bound on the regret of cascading non-stationary bandits. 
	We show that the upper bounds match this lower bound up to a logarithmic factor. 
	\item We evaluate the performance of $\cascadeswucb$ and  $\cascadeducb$ empirically on a real-world web search click dataset. 
\end{inparaenum}

%!TEX root = Paper.tex

\section{Background}
\label{sec:background}

We define the learning problem at the core of this paper in terms of cascading non-stationary bandits.
Their definition builds on the \ac{CM} and its online variant \emph{cascading bandits}, which we review in this section.

We write $[n]$ for $\set{1, \dots, n}$. 
For sets $A$ and $B$, we write $A^B$ for the set of all vectors whose entries are indexed by $B$ and take values from $A$. 
We use boldface letters to denote random variables. We denote a set of candidate items by $\cD = [L]$, e.g., a set of preselected documents. 
The presented ranked list is denoted as $\cR \in \Pi_K(\cD)$, where $\Pi_K(\cD)$ denotes the set of all possible  combinations of $K$ distinct items from $\cD$. 
The item at position $k$ in $\cR$ is denoted as $\cR(k)$, and the position of item $a$ in $\cR$ is denoted as $\cR^{-1}(a)$ 

\subsection{Cascade Model}
\label{sec:cascade model}

We refer readers to \cite{chuklin2015click} for an introduction to click models. 
Briefly, a click model models a user's interaction behavior with the search system. 
The user is presented with a $K$-item ranked list $\cR$.  
Then the user browses the list $\cR$ and clicks items that potentially attract him or her. 
Many click models have been proposed and each models a certain aspect of interaction behavior.
We can parameterize a click model by attraction probabilities $\alpha \in [0, 1]^L$ and a click model assumes:
\begin{assumption}
\label{ass:1}
The attraction probability $\alpha(a)$ only depends on item $a$ and is independent of other items.
 \end{assumption}
 
\noindent%
\ac{CM} is a widely-used click model~\cite{craswell08experimental,batchrank}. 
%A \ac{CM} is parameterized by attraction probabilities $\alpha \in [0,  1]^\cD$. 
In the \ac{CM}, a user browses the ranked list $\cR$ from the first item $\cR(1)$ to the last item $\cR(K)$, which is called the \emph{cascading assumption}. 
After the user browses an item $\cR(i)$, he or she clicks on $\cR(i)$ with attraction probability $\alpha(\cR(i))$, and then stops browsing the remaining items. 
Thus, the examination probability of item $\cR(j)$ equals the probability of no click on the higher ranked items: $\prod_{i=1}^{j-1} (1-\alpha(\cR(i)))$. 
The expected number of clicks equals the probability of clicking any item in the list: $1-\prod_{i=1}^{K} (1-\alpha(\cR(i)))$. 
Note that the reward does not depend on the order in $\cR$, and thus, in the \ac{CM}, the goal of ranking is to find the $K$ most attractive items. 

The \ac{CM} accepts at most one click in each search session.  
It cannot explain scenarios where a user may click multiple  items. 
The \ac{CM} has been extended in different ways to capture multi-click cases~\cite{chapelle09dynamic,guo09efficient}. 
Nevertheless, \ac{CM} is still the fundamental click model and fits historical click data reasonably well.  
Thus, in this paper, we focus on the \ac{CM} and in the next section we introduce  an online variant of \ac{CM}, called \emph{cascading bandits}.

\subsection{Cascading Bandits}
\label{sec:cascadingbandits}
\emph{Cascading bandits} (\acs{CB})\acused{CB} is defined  by a tuple $B = (\cD, P, K)$, where $\cD = [L]$ is the set of candidate items, $K \leq L$ is the number of positions, $P \in \{0, 1\}^{L}$ is a distribution over binary attractions.

In \ac{CB}, at time $t$, a learning agent builds a ranked list $\vR_t \in \Pi_K(\cD)$ that depends on the historical observations up to $t$ and shows it to the user. 
$\vA_t \in \{0, 1\}^L$ is defined as the \emph{attraction indicator}, which is drawn from $P$ and $\vA_t(\vR_t(i))$ is the attraction indicator of item $\vR_t(i)$. 
The user examines $\vR_t$ from $\vR_t(1)$  to $\vR_t(K)$ and clicks the first attractive item.
Since a \ac{CM} allows at most one click each time, a random variable $\vc_t$ is used to indicate the position of the clicked item, i.e., $\vc_t = \argmin_{i\in [K]}  \mathds{1}\{\vA_t(\vR_t(i))\}$. 
If there is no attractive item, the user will not click, and we set $\vc_t = K+1$ to indicate this case. 
Specifically, if $\vc_t \leq K$, the user clicks an item, otherwise, the user does not click anything. 
After the click or browsing the last item in $\vR_t$, the user leaves the search session. 
The click feedback $\vc_t$ is then observed by the learning agent.
Because of the cascading assumption, the agent knows that items ranked above position $\vc_t$ are observed. 
The reward at time $t$ is defined by the number of clicks:
\begin{equation}
r(\vR_t, \vA_t) = 1 - \prod_{i=1}^{K} (1-\vA_t(\vR_t(i)))\,.
\end{equation}

\noindent%
Under \cref{ass:1},  the attraction indicators of each item in $\cD$ are independently distributed.
Moreover, cascading bandits make another assumption. 
\begin{assumption}
	The attraction indicators are distributed as:
	\begin{equation}
	P(\vA) = \prod_{a \in \cD} P_a(\vA(a))\,,
	\end{equation}
	where $P_a$ is a Bernoulli distribution with a mean of $\alpha(a)$.
	\label{ass:2}
\end{assumption}

\noindent%
Under \cref{ass:1} and~\ref{ass:2}, $\vA_t(a)$, the attraction indicator of item $a$ at time $t$, is drawn independently from other items. 
Thus, the expectation of reward of the ranked list at time $t$ is computed as $\expect{r(\vR_t, \vA_t)} = r(\vR_t,  \alpha)$.  
And the goal of the agent is to maximize the expected number of clicks in $n$ steps. 
%, which is equivalent to minimizing the $n$-step cumulative regret: 
%\begin{equation}
%R(n) = \sum_{t=1}^{n}\expect{\max_{R \in \Pi_K(\cD)} r(R, \alpha) - r(\vR_t, \alpha)}.
%\end{equation}

Cascading bandits are designed for a stationary environment, where the attraction probability $P$ remains constant. 
However, in real-world applications, users change their preferences constantly~\cite{jagerman-when-2019}, which is called a \emph{non-stationary environment}, and learning algorithms proposed for cascading bandits, e.g., $\cascadeklucb$ and $\cascadeucb$~\cite{kveton15cascading}, may have linear regret in this setting. 
In the next section, we propose cascading non-stationary bandits, the first non-stationary variant of cascading bandits, and then propose two algorithms for solving this problem.

%!TEX root = Paper.tex

\section{Cascading Non-Stationary Bandits}
\label{sec:cascading non-stationary bandits}
We first define our non-stationary online learning setup, and then we propose two algorithms learning in this setup. 

\subsection{Problem Setup}
\label{sec:setup}

The learning problem we study is called \emph{cascading non-stationary bandits}, a variant of \ac{CB}. 
We define it by a tuple $B = (\cD, P, K, \Upsilon_n)$, where $\cD = [L]$ and $K \leq L$ are the same as in \ac{CB} bandits, $P \in \{0, 1\}^{n\times L}$ is a distribution over binary attractions and $\Upsilon_n$ is the number of abrupt changes in $P$ up to step $n$.
We use $P_t(\vR_t(i))$ to indicate the attraction probability distribution of item $ \vR_t(i)$ at time $t$. 
If $\Upsilon_n = 0$, this setup is  same as \ac{CB}.
The difference is that  we consider a non-stationary learning setup in which $\Upsilon_n > 0$ and the non-stationarity in attraction probabilities characterizes our learning problem.

In this paper, we consider an abruptly changing environment, where  the attraction probability $P$ remains constant within an epoch but can change at any unknown moment in time and the number of abrupt changes up to $n$ steps is $\Upsilon_n$.
The learning agent interacts with cascading non-stationary bandits in the same way as with \ac{CB}. 
Since the agent is in a non-stationary environment, we write $\alpha_t$ for the mean of the attraction probabilities at time $t$ and we evaluate the agent by the expected cumulated regret expressed as:  
\begin{equation}
	R(n) = \sum_{t=1}^{n}\expect{\max_{R \in \Pi_K(\cD)} r(R, \alpha_t) - r(\vR_t, \vA_t)}.
\end{equation}
The goal of the agent it to minimizing the $n$-step regret. 

\subsection{Algorithms}
\label{sec:algorithms}

\begin{algorithm}[t]
	\caption{UCB-type algorithm for Cascading non-stationary bandits.}
	\label{alg:cascadingbandits}
	\begin{algorithmic}[1]
		\STATE \textbf{Input}: discounted factor $\gamma$ or sliding window size $\tau$
		\vspace{0.05in}
		
		\STATE // Initialization
		\STATE \label{alg:initial1} $\forall a \in \cD: \vN_0(a) = 0$
		\STATE \label{alg:initial2} $\forall a \in \cD: \vX_0(a) = 0$
		
		\vspace{0.05in}
		\FOR{$t = 1, 2, \ldots, n$}
		\FOR{$a \in \cD$}
		\STATE // Compute UCBs
		\STATE \label{alg:ucbs}
		$\vU_t(a) \leftarrow  
		\begin{cases}
		\text{Eq. \ref{eq:ducb}} & (\cascadeducb)\\
		\text{Eq. \ref{eq:swucb}} & (\cascadeswucb)
		\end{cases}
		$ 
		\ENDFOR
		\vspace{0.05in}
		
		\STATE // Recommend top $K$ items and receive clicks
		\STATE $\vR_t \leftarrow \argmax_{\vR\in\Pi_K(\cD)} r(\vR, \vU_t)$
		\STATE Show $\vR_t$ and receive clicks $\vc_t \in \{1, \ldots, K+1 \}$
		\vspace{0.05in}
		
		\STATE // Update statistics
		\IF{$\cascadeducb$}  \label{alg:statistics1}
		\STATE // for $\cascadeducb$
		\STATE \label{alg:ducb1} $\forall a \in \cD: \vN_t(a) =  \gamma \vN_{t-1}(a)$
		\STATE \label{alg:ducb2} $\forall a \in \cD: \vX_t(a) =  \gamma \vX_{t-1}(a)$
		\ELSE% \COMMENT{for $\cascadeswucb$}
		\STATE // for $\cascadeswucb$
		\STATE\label{alg:swucb1} $\forall a \in \cD: \vN_t(a) =  \sum_{s=t-\tau +1}^{t-1}
		\mathds{1}\{a \in \vR_s\}$
		\STATE \label{alg:swucb2}$\forall a \in \cD: \vX_t(a) =  \sum_{s=t-\tau +1}^{t-1}
		\mathds{1}\{\vR_s^{-1}(a) = \vc_s\}$
		
		\ENDIF
		\FOR{$i = 1, \ldots, \min\{\vc_t, K\}$}
		\STATE $a \leftarrow \vR_t(i)$
		\STATE  $\vN_t(a) =  \vN_t(a)  + 1 $
		\STATE  $\vX_t(a) = \vX_t(a) + \mathds{1}\{i = \vc_t \}$
		\label{alg:statistics2}
		\ENDFOR
		\ENDFOR
	\end{algorithmic}
\end{algorithm}

We propose two algorithms for solving cascading non-stationary bandits, $\cascadeducb$ and $\cascadeswucb$. 
$\cascadeducb$ is inspired by $\ducb$ and $\cascadeswucb$ is inspired by $\swucb$~\cite{garivier-2011-nonstationary}.
We summarize the pseudocode of both algorithms in~\cref{alg:cascadingbandits}.

$\cascadeducb$ and  $\cascadeswucb$ learn in a similar pattern. 
They differ in the way they estimate the {Upper Confidence Bound} (UCB) $\vU_t(\vR_t(i))$ of the attraction probability of item $\vR_t(i)$ as time $t$, as discussed later in this section. 
After estimating the UCBs (line~\ref{alg:ucbs}), both algorithms construct $\vR_t$ by including the top $K$ most relevant items by UCB. 
Since the order of top $K$ items only affects the observation but does not affect the payoff of  $\vR_t$, we construct $\vR_t$ as follows:
\begin{equation}
 \vR_t = \argmax_{\cR \in \Pi_K(\cD)} r(\cR, \vU_t). 
\end{equation}
After receiving the user's click feedback $\vc_t$, both algorithms update their statistics (line~\ref{alg:statistics1}--\ref{alg:statistics2}). 
We use $\vN_t(i)$ and $\vX_t(i)$  to indicate the number of items $i$ that have been observed and clicked up to $t$ step, respectively.

To tackle the challenge of non-stationarity, $\cascadeducb$  penalizes old observations with a discount factor $\gamma \in (0, 1)$. 
Specifically, each of the previous statistics is discounted by $\gamma$ (line \ref{alg:ducb1}--\ref{alg:ducb2}).
The UCB of item $a$ is  estimated as: 
\begin{equation}
\label{eq:ducb}
	\vU_t(a) = \bar{\valpha}_t(\gamma, a)+ c_t(\gamma, a),
\end{equation}
where $\bar{\valpha}_t(\gamma, a) = \frac{\vX_t(a)}{\vN_t(a)}$ is the  average of discounted attraction indicators  of item $i$ and
\begin{equation}
c_t(\gamma, a) = 2\sqrt{\frac{\epsilon\ln{N_t(\gamma)}}{\vN_t(a)}}
\end{equation}
is the confidence interval around $\bar{\valpha}_t(i)$ at time $t$. 
Here, we compute $N_t(\gamma) = \frac{1 - \gamma^t}{1 - \gamma}$ as the discounted time horizon. 
As shown in \cite{garivier-2011-nonstationary}, $\alpha_t(a) \in [ \bar{\valpha}_t(\gamma, a) - c_t(\gamma, a) ,  \bar{\valpha}_t(\gamma, a)+ c_t(\gamma, a)] $ holds with high probability.

As to $\cascadeswucb$, it estimates UCBs by observations inside a sliding window with size $\tau$. 
Specifically, it only considers the observations in the previous $\tau$ steps~(line~\ref{alg:swucb1}--\ref{alg:swucb2}). 
The UCB of item $i$ is estimated as
\begin{equation}
\label{eq:swucb}
\vU_t(a) = \bar{\valpha}_t(\tau, a)+ c_t(\tau, a),
\end{equation}
where $\bar{\valpha}_t(\tau, a)=\frac{\vX_t(a)}{\vN_t(a)} $ is the  average of observed attraction indicators of item $a$ inside the sliding window and
\begin{equation}
c_t(\tau, a) = \sqrt{\frac{\epsilon\ln{( t \land \tau)}}{\vN_t(a)}} 
\end{equation}
is the confidence interval, and $t \land \tau = \min(t, \tau)$.

\paragraph{Initialization.}
%One way to initialize a UCB-based bandit algorithm is to pull each arm once~\cite{cascadingbandit}. 
%In our set, this can be achieved by showing each item on the top position once and take $O(L)$ steps. 
In the initialization phase, we set all the statistics to $0$ and define  $\frac{x}{0}:=1$ for any $x$ (line~\ref{alg:initial1}--\ref{alg:initial2}).
Mapping back this to UCB, at the beginning, each item has the optimal assumption on the attraction probability with an optimal bonus on uncertainty. 
This is a common initialization strategy for UCB-type bandit algorithms~\cite{Li2018mergedts}.

\section{Analysis}
\label{sec:analysis}

In this section, we analyze the $n$-step regret of $\cascadeducb$ and $\cascadeswucb$. 
We  first derive regret upper bounds on $\cascadeducb$ and $\cascadeswucb$, respectively.
Then we  derive a regret lower bound on cascading non-stationary bandits.
Finally, we discuss our theoretical  results.

\subsection{Regret Upper Bound}
\label{sec:upper bound on cascadeducb}
We refer to $\cDst \subseteq [L]$ as the set of the $K$ most attractive  items in set $\cD$ at time $t$ and $\bar{\cD}_t$ as the complement of $\cDst$, i.e.,  $\forall a \in \cD_t^*\,, \forall a^* \in \bar{\cD}_t: \alpha_t(a) \geq \alpha_t(a^*)$ and $\cDst \cup \cDct = \cD\,, \cDst \cap \cDct = \emptyset$. 
At time $t$, we say an item $a^*$ is optimal if $a^* \in \cDst$ and an item $a$ is suboptimal if $a \in \cDct$. 
The regret at time $t$ is caused by the case that $\vR_t$ includes at least one suboptimal and examined items. 
Let $\Delta_{a, a^*}^t$ be the gap of attraction probability between a suboptimal item $a$ and an optimal $a^*$ at time $t$: $\Delta_{a, a^*}^t = \alpha_t(a^*) - \alpha_t(a)$.
Then we refer to $\Delta_{a, K}$ as the smallest gap of between item $a$ and the $K$-th most attractive item in all $n$ steps when $a$ is not the optimal items: 
$\Delta_{a, K} = \min_{t \in [n], a^* \in \cDst}  \alpha_t(a^*) - \alpha_t(a) $.

\begin{theorem}
	\label{th:upperboundcascadeducb}
	Let $\epsilon \in (1/2, 1)$ and $\gamma \in (1/2, 1)$, the expected $n$-step regret of $\cascadeducb$ is bounded as: 
	\begin{equation}
	\label{eq:upperboundcascadeducb}
	\begin{split}
		&R(n) \leq {}\\
		& L \Upsilon_n\frac{\ln[(1-\gamma)\epsilon]}{\ln\gamma} 
		+ 
		\sum_{a \in \cD}  C(\gamma, a) \lceil n (1-\gamma) \rceil \ln{\frac{1}{1-\gamma}} \,,
	\end{split}
	\end{equation}
	where 
	\begin{equation}
	\begin{split}
	&C(\gamma, a) = {}\\
	&\frac{4}{1-1/e}\ln{(1+4\sqrt{1-1/2\epsilon})} + \frac{32\epsilon}{\Delta_{a, K}\gamma^{1/(1-\gamma)}} \,.
	\end{split}
	\end{equation}
\end{theorem}

\noindent%
We outline the proof in $4$ steps below; the full version is in \cref{sec:proofoftheorem2}.\footnote{\url{https://arxiv.org/abs/1905.12370}}

\begin{proof}
Our proof is adapted from the analysis in \cite{kveton15cascading}. 
The novelty of the proof comes from the fact that, in a non-stationary environment, the discounted estimator $\barvalphat(\gamma, a)$ is now a biased estimator of $\alpha_t(a)$ (\emph{Step 1, 2} and \emph{4}).

\emph{Step 1.} We bound the regret of the event that estimators of the attraction probabilities are biased by $L\Upsilon\frac{\ln[(1-\gamma)\epsilon]}{\ln\gamma}$. 
This event happens during the steps following a breakpoint. 

\emph{Step 2.} We bound the regret of the event that $\alpha_t(a)$ falls outside of the confidence interval around $\barvalphat(\gamma, a)$ by $\frac{4}{1-1/e}\ln{(1+4\sqrt{1-1/2\epsilon})} n(1-\gamma) \ln\frac{1}{1-\gamma}$. 

\emph{Step 3.} We decompose the regret at time $t$ based on \cite[Theorem 1]{kveton15cascading}.

\emph{Step 4.} For each item $a$, we bound the number of times that item $a$ is chosen when $a \in \cDct$ in $n$ steps and get the term $\frac{32\epsilon \lceil n (1-\gamma) \rceil  \ln\frac{1}{1-\gamma}}
{\Delta_{a, K}\gamma^{1/(1-\gamma)}}  $. 
Finally, we sum up all the regret. 
\end{proof}

\noindent%
The bound depends on step $n$ and the number of breakpoints $\Upsilon_n$. 
If they are known beforehand, we can choose $\gamma$ by minimizing the right hand side of Eq.~\ref{eq:upperboundcascadeducb}. 
Choosing $\gamma = 1 - 1/4 \sqrt{(\Upsilon_n/n)}$ leads to $R(n) = O(\sqrt{n\Upsilon_n}\ln n)$. 
When $\Upsilon_n$ is independent of $n$, we have $R(n) = O(\sqrt{n\Upsilon}\ln n)$. 

\begin{theorem}
	\label{th:upperboundcascadeswucb}
	Let $\epsilon \in (1/2, 1)$. For any integer $\tau$, the expected $n$-step regret of $\cascadeswucb$ is bounded as: 
	\begin{equation}
	\label{eq:upperboundcascadeswucb}
	\begin{split}
	\mbox{}\hspace*{-2.5mm}&R(n) \leq \\
	\mbox{}\hspace*{-2.5mm}&L  \Upsilon_n\tau \!+\! \frac{L\ln^2\tau}{\ln(1+4\sqrt{(1-1/2\epsilon)})} \!+\! \sum_{a \in \cD} C(\tau, a) \frac{n\ln \tau}{\tau} ,\mbox{}\hspace*{-2mm}
	\end{split}
	\end{equation}
	where 
	\begin{equation}
	\begin{split}
	&C(\tau, a) = {}\\
	&\frac{2}{\ln\tau} \left\lceil \frac{\ln \tau}{\ln(1+4\sqrt{(1-1/2\epsilon)})}\right\rceil + \frac{8\epsilon}{\Delta_{a, K}} \frac{\lceil n/\tau \rceil}{n/\tau}.
	\end{split}
	\end{equation}
	When $\tau$ goes to infinity and $n/\tau$ goes to $0$, 
	\begin{equation}
	 C(\tau, a) = \frac{2}{\ln(1+4\sqrt{(1-1/2\epsilon)})}+ \frac{8\epsilon}{\Delta_{a, K}} .
	\end{equation}
\end{theorem}

\noindent%
We outline the proof in $4$ steps below and the full version is in \cref{sec:proofoftheorem3}. 

\begin{proof}
The proof follows the same lines as the proof of \cref{th:upperboundcascadeducb}. 

\emph{Step 1.} We bound the regret of the event that estimators of the attraction probabilities are biased by $L  \Upsilon_n\tau$. 

\emph{Step 2.} We bound the regret of the event that $\alpha_t(a)$ falls outside of the confidence interval around $\barvalphat(\tau, a)$ by 
\begin{equation}
\ln^2\tau + 2n \left\lceil \frac{\ln \tau}{\ln(1+4\sqrt{(1-1/2\epsilon)})}\right\rceil. 
\end{equation}

\emph{Step 3.} We decompose the regret at time $t$ based on \cite[Theorem 1]{kveton15cascading}.

\emph{Step 4.} For each item $a$, we bound the number of times that item $a$ is chosen when $a \in \cDct$ in $n$ steps and get the term $ \frac{8\epsilon}{\Delta_{a, K}} \lceil \frac{ n}{\tau}\rceil $. 
Finally, we sum up all the regret. 
\end{proof}

\noindent%
If we know $\Upsilon_n$ and $n$ beforehand, we can choose the window size $\tau$ by minimizing the right hand side of Eq.~\ref{eq:upperboundcascadeswucb}. 
Choosing $\tau = 2\sqrt{n\ln(n) / \Upsilon_n}$ leads to $R(n) = O(\sqrt{n\Upsilon_n\ln n})$. 
When $\Upsilon_n$ is independent of $n$, we have $R(n) = O(\sqrt{n\Upsilon\ln n})$. 

\subsection{Regret Lower Bound}
\label{sec:lower bound}

We consider a particular cascading non-stationary bandit and refer to it as  $B_{L}=(L, K, \Delta, p, \Upsilon)$. 
We have a set of $L$ items  $\cD = [L]$ and $K = \frac{1}{2} L$ positions. 
At any time $t$, the distribution of attraction probability of each item $a\in\cD$ is parameterized by: 
\begin{equation}
\alpha_t(a)= \begin{cases}
	p \qqquad   \text{if }a \in \cDst\\
	p-\Delta  \qquad \text{if }a \in \cDct, 
\end{cases}
\end{equation}
where $ \cDst$ is the set of optimal items at time $t$, $\cDct$ is the set suboptimal items at time $t$,  and $\Delta \in (0, p]$ is the gap between optimal items and suboptimal items.
Thus, the attraction probabilities only take two values: $p$ for optimal items and $p-\Delta$ for suboptimal items up to $n$-step. 
$\Upsilon$ is the number of breakpoints when the attraction probability of an item changes from $p$ to $p-\Delta$ or other way around. 
Particularly, we consider a simple variant that  the distribution of attraction probability of each item is piecewise constant and has two breakpoints. 
And we assume another constraint on the number of optimal items that $|\cDst| = K$ for all time steps $t\in[n]$.
Then, the regret that any learning policy can achieve when interacting with $B_L$ is lower bounded by \cref{th:lowerbound}.

\begin{theorem}
	\label{th:lowerbound}
	The $n$-step regret of any learning algorithm inter\-acting with 
	cascading non-stationary bandit $B_{L}$ is lower bounded as follows:
	\begin{equation}
		\label{eq:lowerbound}
		\liminf\limits_{n\rightarrow \infty} R(n) \geq L \Delta (1-p)^{K-1} \sqrt{\frac{2 n}{3D_{KL}(p-\Delta || p)}}, 
	\end{equation}
	where $D_{KL}(p-\Delta || p)$ is the Kullback-Leibler (KL) divergence between two Bernoulli distributions with means $p-\Delta$ and $p$. 
\end{theorem}

\begin{proof}
The proof is based on the analysis in \cite{kveton15cascading}. 
We first refer to  $\cR_t^*$ as the optimal list at time $t$ that includes $K$ items. 
For any time step $t$, any item $a \in \cDct$ and any item $a^* \in \cDst$, we define the event that item $a$ is included in $\vR_t$ instead of item $a^*$ and  item $a$ is examined but not clicked at time step $t$ by: 
\begin{equation}
\begin{split}
& G_{t, a, a^*} = {}\\
& \quad\{ \exists 1\leq k <\vc_t ~s.t.~ \vR_t(k) = a \,  , \cR_t(k)  = a^*\}.
\end{split}
\end{equation}
By \cite[Theorem 1]{kveton15cascading}, the regret at time $t$ is decomposed as:
\begin{equation}
	\mathbb{E}[r(\vR_t, \valpha_t)] \geq  \Delta (1-p)^{K-1}\sum_{a \in \cDct} \sum_{a* \in \cDst} \mathds{1}\{ G_{a, a^*, t} \}.
\end{equation}
Then, we bound the $n$-step regret as follows:
\begin{equation}
\begin{split}
	R(n) &\geq  \Delta (1-p)^{K-1}  \sum_{t=1}^{n} \sum_{a \in \cDct} \sum_{a^* \in \cDst} \mathds{1}\{G_{t, a, a^*}\} \\ 
	&\geq  \Delta (1-p)^{K-1} \sum_{a \in \cD}\sum_{t=1}^{n} \mathds{1}\{a\in \cDct, a \in\vR_t\}\\
	 &= \Delta (1-p)^{K-1}  \sum_{a \in \cD} \vT_n(a),
\end{split}	 
\end{equation}
where $\vT_n(a) =  \sum_{t=1}^{n}\mathds{1}\{a\in \cDct, a\in\vR_t, \vR_t^{-1}(a) \leq \vc_t \}$. 
The second inequality is based on the fact that, at time $t$,  the event $G_{t, a, a^*}$ happens if and only if item $a$ is suboptimal and examined. 
By the results of \cite[Theorem 3]{garivier-2011-nonstationary}, if a suboptimal item $a$ has not been examined enough times, the learning policy may play this item for a long period after a breakpoint. 
And we get: 
\begin{equation}
		\liminf\limits_{n\rightarrow \infty} \vT(n) \geq  \sqrt{\frac{2 n}{3D_{KL}(p-\Delta || p)}}.
\end{equation}
We sum up all the inequalities and obtain: 
\begin{equation*}
\liminf\limits_{n\rightarrow \infty} R(n) \geq L \Delta (1-p)^{K-1} \sqrt{\frac{2 n}{3D_{KL}(p-\Delta || p)}} . \qedhere
\end{equation*}
\end{proof}

\subsection{Discussion}
\label{sec:discussion}
We have shown that the $n$-step regret upper bounds of $\cascadeducb$  and  $\cascadeswucb$ have the order of $O(\sqrt{n} \ln n)$and $O(\sqrt{n\ln n})$, respectively. 
They match the lower bound proposed in \cref{th:lowerbound} up to a logarithmic factor. 
Specifically, the upper bound of $\cascadeducb$ matches the lower bound up to $\ln n$. 
The upper bound of $\cascadeswucb$ matches the lower bound up to $\sqrt{\ln n}$, an improvement over $\cascadeducb$, as confirmed by experiments in~\cref{sec:experiments}. 

We have assumed that step $n$ is know beforehand. 
This may not always be the case. 
We can extend $\cascadeducb$ and $\cascadeswucb$ to the case where $n$ is unknown by using the \emph{doubling trick}~\cite{garivier-2011-nonstationary}. 
Namely, for $t > n$ and any $k$, such that $2^k \leq t < 2^{k+1}$, we reset $\gamma = 1 - \frac{1}{4\sqrt{2^k}}$ and $\tau = 2\sqrt{2^k\ln(2^k)}$. 

$\cascadeducb$ and $\cascadeswucb$ can be computed efficiently. 
Their complexity is linear in the number of time steps. 
However, $\cascadeswucb$ requires extra memory to remember past ranked lists and rewards to update $\vX_t$ and $\vN_t$.

%!TEX root = Paper.tex

\section{Experimental Analysis}
\label{sec:experiments}

We evaluate $\cascadeducb$ and $\cascadeswucb$ on the \emph{Yandex} click  dataset,\footnote{\url{https://academy.yandex.ru/events/data_analysis/relpred2011}} which is the largest public click collection. 
It contains  more than $30$ million search sessions, each of which  contains at least one search query. 
We process the queries in the same manner as  in~\cite{batchrank,kveton2018bubblerank}. 
Namely, we randomly select $100$ frequent search queries with the $10$ most attractive items in each query, and then learn a CM for each query using PyClick.\footnote{\url{https://github.com/markovi/PyClick}}
These CMs are used to generate click feedback. 
In this setup, the size of  candidate items is $L=10$ and we choose $K=3$ as the number of positions. 
The objective of the learning task is to identify $3$ most attractive items and then maximize the expected number of clicks at the $3$  highest positions. 

We consider a simulated non-stationary environment setup, where we take the learned attraction probabilities as the default and change the attraction probabilities periodically. 
Our simulation can be described in $4$ steps: 
\begin{inparaenum}[(1)]
		\item For each query, the attraction probabilities of the top $3$ items remain constant over time. 
		\item We  randomly choose three suboptimal items and set their attraction probabilities to $0.9$ for the next $m_1$ steps. 
		\item Then we reset the attraction probabilities and keep them constant for the next $m_2$ steps.  
		\item We repeat step (2) and step (3) iteratively.
\end{inparaenum}
This simulation mimics abrupt changes in user preferences and is widely used in previous work on non-stationarity~\cite{garivier-2011-nonstationary,wu2018learning,jagerman-when-2019}. 
In our experiment, we set $m_1 = m_2 = 10$k and choose $10$ breakpoints. 
In total, we run experiments for $100$k steps.
Although the non-stationary aspects in our setup are simulated, the other parameters of a CM are learned from the Yandex click dataset.

We compare $\cascadeducb$ and $\cascadeswucb$,  to $\rankedexpthree$~\cite{radlinski08learning}$, \cascadeklucb$~\cite{kveton15cascading} and $\batchrank$~\cite{batchrank}. 
We describe the baseline algorithms in slightly more details in \cref{sec: related work}. 
Briefly, $\rankedexpthree$, a variant of ranked bandits, is based on an adversarial bandit algorithm $\expthree$~\cite{auer95gambling}; it is the earliest bandit-based ranking algorithm and is popular in practice. 
$\cascadeklucb$~\cite{kveton15cascading} is a near optimal algorithm in CM. 
$\batchrank$~\cite{batchrank} can learn in a wide range of click models. 
However, these algorithms only learn in a stationary environment. 
We choose them as baselines to show the superiority of our algorithms in a non-stationary environment.
In experiments,  we set $\epsilon=0.5$, $\gamma = 1 - 1/(4 \sqrt{n})$ and $\tau= 2\sqrt{n\ln(n)}$, the values that roughly minimize the upper bounds.

\begin{figure}[t]
	\includegraphics{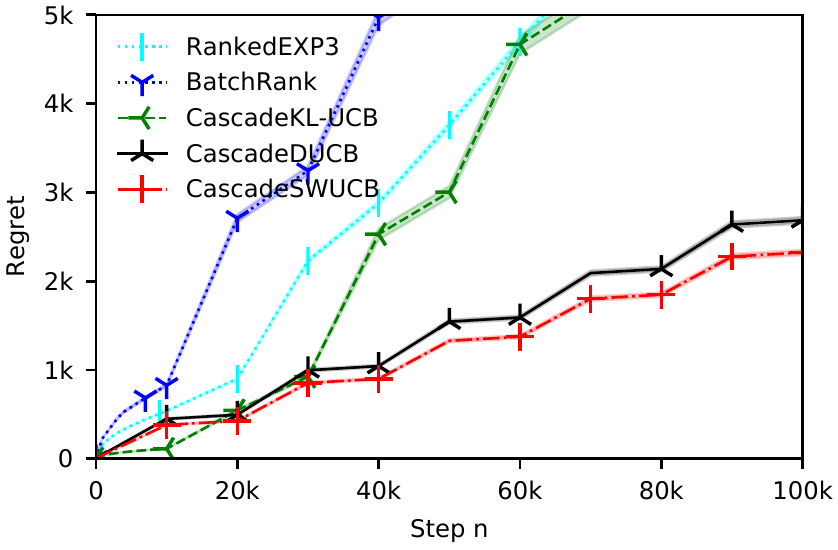}
	\caption{The n-step regret in up to $100$k steps. Lower is better. The results are averaged over all $100$ queries and $10$ runs per query. The shaded regions represent standard errors of our estimates.}\label{fig:regret}
\end{figure}

We report the $n$-step regret averaged over $100$ queries and $10$ runs per query in~\cref{fig:regret}. 
All baselines have linear regret in time step. 
They fail in capturing the breakpoints. 
Non-stationarity makes the baselines perform even worse during epochs where the attraction probability are set as  the default.
E.g., $\cascadeklucb$ has $111.50 \pm 1.12$ regret in the first $10$k steps but has $447.82\pm137.16$ regret between step $80$k and $90$k. 
Importantly, the attraction probabilities equal the default and remain constant inside these two epochs. 
This is caused by the fact that the baseline algorithms rank items based on all historical observations, i.e., they do not balance ``remembering" and ``forgetting." 
Because of the use of a discounting factor and/or a sliding window, $\cascadeducb$ and $\cascadeswucb$ can detect breakpoints and show convergence. 
$\cascadeswucb$ outperforms $\cascadeducb$ with a small gap; this is consistent with our theoretical finding that $\cascadeswucb$ outperforms $\cascadeducb$ by a $\sqrt{\ln n}$ factor.

%!TEX root = Paper.tex

\section{Related Work}
\label{sec: related work}

The idea of directly learning to rank from user feedback has been widely studied in a stationary environment.
\emph{Ranked bandits}~\cite{radlinski08learning} are among the earliest \ac{OLTR} approaches. 
In ranked bandits, each position in the list is modeled as an individual underlying \acp{MAB}. 
The ranking task is then solved by asking each individual \ac{MAB} to recommend an item to the attached position. 
Since the reward, e.g., click, of a lower position is affected by higher positions, the underlying \ac{MAB} is typically adversarial, e.g., $\expthree$~\cite{auer95gambling}.
$\batchrank$ is a  recently proposed \ac{OLTR} method~\cite{batchrank}; it is an elimination-based algorithm: once an item is found to be inferior to $K$ items, it will be removed from future consideration. 
$\batchrank$ outperforms ranked bandits in the stationary environment.
In our experiments, we include  $\batchrank$ and $\rankedexpthree$, the $\expthree$-based ranked bandit algorithm, as baselines.

Several \ac{OLTR} algorithms have been proposed in specific click models~\cite{kveton15cascading,pbmbandit,katariya16dcm,oosterhuis-differentiable-2018}. 
They can efficiently learn an optimal ranking given the click model they consider. 
Our work is related to cascading bandits and we compare our algorithms to $\cascadeklucb$, an algorithm proposed for soling cascading bandits~\cite{kveton15cascading}, in \cref{sec:experiments}. 
Our work differs from cascading bandits in that we consider learning in a non-stationary environment.

Non-stationary bandit problems have been widely studied~\cite{slivkins2008adapting,yu2009piecewise,garivier-2011-nonstationary,besbes2014stochastic,liu2018change}.  
However, previous work requires a small action space. 
In our setup, actions are (exponentially many) ranked lists.
Thus, we do not consider them as baselines in our experiments. 

In \emph{adversarial bandits} the reward realizations, in our case attraction indicators, are selected by an \emph{adversary}. 
Adversarial bandits originate from game theory~\cite{blackwell1956analog} and have been widely studied, cf.~\cite{auer95gambling,cesa2006prediction} for an overview. 
Within  adversarial bandits, the performance of a policy is often measured by comparing to a static oracle which always chooses a single best arm that is obtained after seeing all the reward realizations up to step $n$. 
This static oracle can perform poorly in a non-stationary case when the single best arm is suboptimal for a long time between two breakpoints.
Thus, even if a policy performs closely to the static oracle, it can still perform sub-optimally in a non-stationary environment. 
Our work differs from adversarial bandits in that we compare to a dynamic oracle that can balance the dilemma of ``remembering" and ``forgetting" and chooses the per-step best action.

%!TEX root = Paper.tex

\section{Conclusion}
\label{sec:conclusion}

In this paper, we have studied the \acf{OLTR} problem in a non-stationary environment where user preferences change abruptly. 
We focus on a widely-used user click behavior model \acf{CM} and have proposed an online learning variant of it called \emph{cascading non-stationary bandtis}.
Two algorithms, $\cascadeducb$ and $\cascadeswucb$, have been proposed for solving it.  
Our theoretical  have shown that they have sub-linear regret.  
These theoretical findings have been confirmed by our experiments on the Yandex click dataset. 
We open several future directions for non-stationary \ac{OLTR}. 
First, we have only considered the \ac{CM} setup.  
Other click models that can handle multiple clicks, e.g.,  DBN~\cite{chapelle09dynamic}, should be considered as part of future work. 
Second, we focused on $\ucb$-based policy. 
Another possibility is to use the family of softmax policies~\cite{besbes2014stochastic}. 
Along this line, one may obtain upper bounds independent from the number of breakpoints.

\if0
First, we have only considered the \ac{CM} setup. 
Although a \ac{CM} is powerful in explanaining user behavior, it can only learn up to the first click, which may ignore part of user feedback. 
Other click models that can handle multiple clicks such as DCM~\cite{guo09efficient} and DBN~\cite{chapelle09dynamic} may be considered as part of future work. 
We believe that our analysis can be adapted to those cases easily. 
Second, we focus on $\ucb$-based policy in building rankings. 
Another possibility is to use the family of softmax policies, which has original been designed for adversarial bandits~\cite{auer2002nonstochastic,besbes2014stochastic}.
Along this line, one may obtain upper bounds independent from the number of breakpoints. 
\fi

\section*{Acknowledgements}
This research was supported by
Ahold Delhaize,
the Association of Universities in the Netherlands (VSNU),
the Innovation Center for Artificial Intelligence (ICAI),
and
the Netherlands Organisation for Scientific Research (NWO)
under pro\-ject nr.\ 612.\-001.\-551.
All content represents the opinion of the authors, which is not necessarily shared or endorsed by their respective employers and/or sponsors.

\clearpage
\bibliographystyle{named}
\bibliography{Bibliography} 

%!TEX root = Paper.tex

\clearpage
\onecolumn
\appendix

%\section{Appendix to IJCAI-19 Submission \# 3077  \\ Cascading Non-Stationary Bandits: Online Learning to Rank in the Non-Stationary Cascade Model}
\section{Proofs}

In the appendix, we refer to  $\cR_t^*$ as the optimal list at time $t$ that includes $K$ items sorted by the decreasing order of their attraction probabilities. 
We refer to $\cDst \subseteq [L]$ as the set of the $K$ most attractive  items in set $\cD$ at time $t$ and $\bar{\cD}_t$ as the complement of $\cDst$, i.e.  $\forall a^* \in \cD_t^*, \forall a \in \bar{\cD}_t: \alpha_t(a^*) \geq \alpha_t(a)$ and $\cDst \cup \cDct = \cD, \cDst \cap \cDct = \emptyset$. 
At time $t$, we say an item $a^*$ is optimal if $a^* \in \cDst$ and an item $a$ is suboptimal if $a \in \cDct$. 
We denote $\vr_t =\max_{R \in \Pi_K(\cD)} r(R, \alpha_t) - r(\vR_t, \vA_t)$ to be the regret at time $t$ of the learning algorithm. 
Let $\Delta_{a, a^*}^t$ be the gap of attraction probability between a suboptimal item $a$ and an optimal $a^*$ at time $t$: $\Delta_{a, a^*}^t = \alpha_t(a^*) - \alpha_t(a)$.
Then we refer to $\Delta_{a, K}$ as the smallest gap between item $a$ and the $K$-th most attractive item in all $n$ time steps when $a$ is not the optimal item: 
$\Delta_{a, K} = \min_{t \in [n], a \notin \cDst}  \alpha_t(K) - \alpha_t(a) $.

\subsection{Proof of Theorem 1}
% upper bound for cascadingducb
\label{sec:proofoftheorem2}

Let $M_{t} = \{\exists a\in\cD~s.t.~|\alpha_t(a) - \barvalphat(\gamma, a)| > c_t(\gamma, t)\}$ be the event that $\alpha_t(a)$ falls out of the confidence interval around $\barvalphat(\gamma, a)$ at time t, and  $\bar{M}_{t}$ be the complement of ${M}_{t}$. 
We re-write the $n$-step regret of $\cascadeducb$ as follows: 
\begin{equation}
\label{eq:R}
R(n)  = \expect{\sum_{t=1}^{n} \mathds{1}\{M_{t}\} \vr_t} + \expect{\sum_{t=1}^{n} \mathds{1}\{\bar{M}_{t}\} \vr_t}. 
\end{equation}
We then bound both terms above separated. 

We refer to $\cT$ as the set of all time steps such that for $t \in \cT$, $s \in [t - B(\gamma), t]$ and any item $a\in\cD$ we have $\alpha_s(a) = \alpha_t(a)$, where $B(\gamma) = \lceil \log_{\gamma}(\epsilon (1-\gamma)) \rceil$. 
In other words, $\cT$ is the set of time steps that do not follow too close after breakpoints. 
Since for any time step $t \notin \cT$ the estimators of attraction probabilities are biased, $\cascadeducb$ may recommend suboptimal items constantly. Thus, we get the following bound:
\begin{equation}
\label{eq:ducbdecompose}
	 \expect{\sum_{t=1}^{n} \mathds{1}\{M_{t}\} \vr_t}  \leq 
	 L\Upsilon_n B(\gamma) + \expect{\sum_{t\in\cT} \mathds{1}\{M_{t}\} \vr_t}.
\end{equation}
Then, we show that for steps $t \in \cT$, the attraction probabilities are well estimated for all items with high probability. 
For an item $a$, we consider the bias and variance of $\barvalphat(\gamma, a)$ separately. 
We denote:  
\[	\vX_t(\gamma, a) = \sum_{s=1}^{t} \gamma^{(t-s)} 		  	\mathds{1}\{a\in\vR_s, \vR_s(\vc_s)=a\}, \quad
	\vN_t(\gamma, a) = \sum_{s=1}^{t} \gamma^{(t-s)} \mathds{1}\{a\in\vR_s\},
\]
as the discounted number of being clicked, and the discounted number of being examined.

First, we  consider the bias. 
The ``bias" at time $t$ can be written as $x_t(\gamma, a)/\vN_t(\gamma, a) - \alpha_t$, where $x_t(\gamma, a) = \sum_{s=1}^{t} \gamma^{(t-s)} \mathds{1}\{a\in\vR_s\} \alpha_s(a)$. 
For $t\in\cT$: 
\begin{eqnarray*}
	|x_t(\gamma, a) - \alpha_t(a) \vN_t(\gamma, a)| &=& \abs{\sum_{s=1}^{t-B(\gamma)} \gamma^{(t-s)}(\alpha_s(a)-\alpha_t(a)\mathds{1}\{a \in \vR_t\}} \\ 
	&\leq& \sum_{s=1}^{t-B(\gamma)}\gamma^{(t-s)} \abs{(\alpha_s(a)-\alpha_t(a)}\mathds{1}\{a \in \vR_t\}\\
	&\leq& \gamma^{B(\gamma)} \vN_{t-B(\gamma)}(\gamma, a)
	\leq  \gamma^{B(\gamma)} \frac{1}{1-\gamma},
\end{eqnarray*}
where the last inequality is due to the fact that $\vN_{t-B(\gamma)}(\gamma, a) \leq 1/(1-\gamma)$. 
Thus, we get:
\begin{eqnarray*}
	\abs{\frac{x_t(\gamma, a)}{\vN_t(\gamma, a)} - \alpha_t(a)} 
	&\leq& \frac{\gamma^{B(\gamma)}}{(1-\gamma)\vN_t(\gamma, a)}
	\leq \left(1 \land  \frac{\gamma^{B(\gamma)}}{(1-\gamma)\vN_t(\gamma, a)}\right) \\
	&\leq& 
	\sqrt{\frac{\gamma^{B(\gamma)}}{(1-\gamma)\vN_t(\gamma, a)}}  
	\leq \sqrt{\frac{\epsilon \ln N_t(\gamma)}{\vN_t(\gamma, a)}} \leq \frac{1}{2} c_t(\gamma, a) , 
\end{eqnarray*} 
where the third inequality is due to the fact that $1\land x \leq \sqrt{x}$ and the last inequality is due to the definition of $B(\gamma)$. 
So, $B(\gamma)$ time steps after a breakpoint, the ``bias" is at most half as large as the confidence bonus;
and the second half of the confidence interval is used for the variance. 

Second, we consider the variance. 
For $a\in \cD$ and $t\in \cT$, let $M_{t,a} = \{|\alpha_t(a) - \barvalphat(\gamma, a)| > c_t(\gamma, t)\}$ be the event that $\alpha_t(a)$ falls out of the confidence interval around $\barvalphat(\gamma, a)$ at time $t$. 
By using a Hoeffding-type inequality (\cite[Theorem 4]{garivier-2011-nonstationary}), for an item $a\in\cD$,  $t\in \cT$, and any $\eta > 0$, we get:
\begin{eqnarray*}
P(M_{t,a}) &\leq& 
P\left( \frac{\vX_t(\gamma, a) - x_t(\gamma, a)}{\sqrt{\vN_t(\gamma^2, a)}} > \sqrt{\frac{\epsilon \ln N_t(\gamma)}{\vN_t(\gamma^2, a)}}\right) 
\\
&\leq& 
P\left( \frac{\vX_t(\gamma, a) - x_t(\gamma, a)}{\sqrt{\vN_t(\gamma^2, a)}} > \sqrt{\epsilon \ln N_t(\gamma)}\right) 
\\
&\leq&
\left\lceil \frac{\ln{N_t(\gamma)}}{\ln{(1+\eta)}} \right\rceil 
N_t(\gamma)^{-2\epsilon(1-\frac{\eta^2}{16})} .
\end{eqnarray*}
Thus, we get the following bound: 
\[
\expect{\sum_{t\in\cT} \mathds{1}\{M_{t}\} \vr_t} 
\leq 
2L\sum_{t\in\cT} 
\left\lceil \frac{\ln{N_t(\gamma)}}{\ln{(1+\eta)}} \right\rceil 
N_t(\gamma)^{-2\epsilon(1-\frac{\eta^2}{16})}.
\]
By taking $\eta = 4\sqrt{1-1/2\epsilon}$ such that  $1-\frac{\eta^2}{16} = 1$, and with $t_0 = (1-\gamma)^{(-1)}$
we get: 
\begin{eqnarray*}
\sum_{t\in\cT} 
\left\lceil \frac{\ln{N_t(\gamma)}}{\ln{(1+\eta)}} \right\rceil 
N_t(\gamma)^{-2\epsilon(1-\frac{\eta^2}{16})} 
&\leq& 
t_0 + \sum_{t\in\cT, t \geq t_0} 
\left\lceil \frac{\ln{N_{t_0}(\gamma)}}{\ln{(1+\eta)}} \right\rceil 
N_{t_0}(\gamma)^{-1}  
\\
&\leq& 
t_0 +
\left\lceil \frac{\ln{N_{t_0}(\gamma)}}{\ln{(1+\eta)}} \right\rceil 
\frac{n}{N_{t_0}(\gamma)} 
\\
&\leq& 
\frac{1}{1-\gamma} +
\left\lceil \frac{\ln{N_{t_0}(\gamma)}}{\ln{(1+\eta)}} \right\rceil 
\frac{n(1-\gamma)}{1-\gamma^{1/(1-\gamma)}} .
\end{eqnarray*}
We sum up and get the upper bound: 
\begin{equation}
	\label{eq:bound for outside}
	\expect{\sum_{t=1}^{n} \mathds{1}\{\bar{M}_{t}\} \vr_t} 
	\leq 
	L \Upsilon_n B(\gamma) 
	+
	2L\frac{1}{1-\gamma} +
	2L\left\lceil \frac{\ln{N_{t_0}(\gamma)}}{\ln{(1+\eta)}} \right\rceil 
	\frac{n(1-\gamma)}{1-\gamma^{1/(1-\gamma)}} .
\end{equation}

%%%%%%%%%%%%%%%%%%%%%%%%%%%%%%%
% last part
%%%%%%%%%%%%%%%%%%%%%%%%%%%%%%%

Third, we upper bound the second term in Eq.~\ref{eq:R}. 
The regret is caused by recommending a suboptimal item to the user and the user examines but does not click the item.
Since there are $\Upsilon_n$ breakpoints, we refer to $[t_1, \ldots, t_{\Upsilon_n}]$ as the time step of a breakpoint that occurs. 
We consider the time step in the individual epoch that does not contain a breakpoint. 
For any epoch and any time $t \in \{t_e, t_e+1, \ldots, t_{e+1}-1\}$, any item $a \in \cDci$ and any item $a^* \in \cDsi$, we define the event that item $a$ is included in $\vR_t$ instead of item $a^*$, and  item $a$ is examined but not clicked at time $t$ by: 
\begin{equation*}
\begin{split}
G_{t, a, a^*} =\{ \exists 1\leq k <\vc_t ~s.t.~ \vR_t(k) = a, \cR_t(k)  = a^*\}.
\end{split}
\end{equation*} 
Since the attraction probability remains constant in the epoch, we refer to $\cDsi$ as the optimal items and $\cDci$ as the suboptimal items in epoch $e$. 
By \cite[Theorem 1]{kveton15cascading}, the regret at time $t$ is decomposed as:
\begin{equation}
\label{eq:ducbrdecompose}
\mathbb{E}[\vr_t] \leq  \Delta^t_{a, a^*}\sum_{a \in \cDci} \sum_{a* \in \cDsi} \mathds{1}\{ G_{a, a^*, t} \} .
\end{equation}
Then we have:
\begin{equation}
\label{eq:inside}
\expect{\sum_{t=t_i}^{t_{i+1}-1} \mathds{1}\{\bar{M}_{t}\} \vr_t} 
= 
\sum_{t=t_i}^{t_{i+1}-1} \mathds{1}\{\bar{M}_{t}\} \expect{\vr_t}
\leq 
\sum_{a \in \cDci} \expect{\sum_{a^*\in\cDsi} \sum_{t=t_i}^{t_{i+1}-1}  
\Delta^t_{a, a^*} \mathds{1}\{ G_{a, a^*, t} \}} ,
\end{equation}
where the first equality is due to the tower rule, and the inequality is due to Eq.~\ref{eq:ducbrdecompose}. 

Now, for any suboptimal item $a$ in epoch $e$, we upper bound $ \expect{\sum_{a^*\in\cDsi} \sum_{t=t_i}^{t_{i+1}-1}  
	\Delta^t_{a, a^*} \mathds{1}\{ G_{a, a^*, t} \}} $. 
At time $t$,  event $\mathds{1}\{\bar{M}_{t}\}$ and event $ \mathds{1}\{a \in \vR_t, a \in \cDct\}$ happen when there exists an optimal item $a^*\in \cDsi$ such that:
 \[
 \alpha_t(a) + 2 c_t(\gamma, a) \geq \vU_t(a) \geq \vU(a^*) \geq \alpha_t(a^*) , 
\]
which implies that $2 c_t(\gamma, a) \geq \alpha_t(a^*) - \alpha_t(a)$. 
Taking the definition of the confidence interval, we get: \[\vN_t(\gamma, a) \leq \frac{16\epsilon \ln N_{t}(\gamma)}{\Delta^{2}_{t, a, a^*}},\] 
where we set $\Delta_{t, a, a^*} = \Delta^{t}_{ a, a^*}$.

Together with Eq.~\ref{eq:inside}, we get: 
\begin{eqnarray}
\expect{
	\sum_{t=t_i}^{t_{i+1}-1} \mathds{1}\{\bar{M}_{t}\} \vr_t
} 
	&\leq& 
	\sum_{a \in \cDci} \expect{\sum_{a^*\in\cDsi} 
	 \frac{16\epsilon \ln N_{t}(\gamma)}{\Delta_{t, a, a^*}}}\notag
	 \\
	&\leq&
	16\epsilon \ln N_{t}(\gamma)
	 \left[
	 	\Delta_{t, a, 1}\frac{1}{\Delta^2_{t, a, 1}} + \sum_{a^*=2}^{K} \Delta_{t, a, a^*} \left( \frac{1}{\Delta^2_{t, a, a^*}} - \frac{1}{\Delta^2_{t, a, a^*-1}} \right)
	 \right] \notag
	 \\
	 &\leq& 
	 \frac{32\epsilon \ln N_{t}(\gamma)}{\Delta_{t, a, K}} ,
\end{eqnarray}
where the last inequality is due to \cite[Lemma 3]{kveton2014matroid}. 
Let $\Delta_{a, K} = \min_{t \in [n]} \Delta_{t, a, K}$ be the smallest gap between the suboptimal item $a$ and an optimal item in all time steps.  
When $\vN_t(\gamma, a) > \frac{32\epsilon \ln N_t(\gamma)}{\Delta^2_{a, K}}$, $\cascadeducb$ will not select item $a$ at time $t$.
Thus we get: 
\begin{eqnarray}
\label{eq:ducbinside2}
	\sum_{a \in \cD} \expect{ \sum_{t=1}^{n} \mathds{1}\{\bar{M}_{t}\} \mathds{1}\{a \in \vR_t, a \in \cDct\}} &\leq& 
	\sum_{e\in [\Upsilon_n]} \sum_{a \in \cDci} \frac{32\epsilon \ln N_{t}(\gamma)}{\Delta_{t, a, K}}   \notag
	\\
	&\leq&
	\sum_{a \in \cD} \lceil n (1-\gamma) \rceil 
	 \frac{32\epsilon \ln N_n(\gamma)}{\Delta_{a, K}}
	\gamma^{1/(1-\gamma)} ,
\end{eqnarray}
where the last inequality is based on \cite[Lemma 1]{garivier-2011-nonstationary}.

Finally, together with Eq.~\ref{eq:R}, Eq.~\ref{eq:ducbdecompose}, Eq.~\ref{eq:bound for outside}, Eq.~\ref{eq:ducbrdecompose}, 
Eq.~\ref{eq:inside} and Eq.~\ref{eq:ducbinside2}, we get \cref{th:upperboundcascadeducb}.

%%%%%%%%%%%%%%%%%%%%%%%%%%%%%%%%%%%%%
\subsection{Proof of Theorem 2} 
% upper bound for cascadingswucb
\label{sec:proofoftheorem3}

Let $M_{t} = \{\exists a\in\cD~s.t.~|\alpha_t(a) - \barvalphat(\tau, a)| > c_t(\tau, t)\}$ be the event that $\alpha_t(a)$ falls out of the confidence interval around $\barvalphat(\tau, a)$ at time t, and  let $\bar{M}_{t}$ be the complement of ${M}_{t}$.  
We re-write the $n$-step regret of $\cascadeswucb$ as follows: 
\begin{equation}
\label{eq:swucbR}
R(n)  = \expect{\sum_{t=1}^{n} \mathds{1}\{M_{t}\} \vr_t} + \expect{\sum_{t=1}^{n} \mathds{1}\{\bar{M}_{t}\} \vr_t} . 
\end{equation}
We then bound both terms in Eq.~\ref{eq:swucbR} separately. 

First, we refer to $\cT$ as the set of all time steps such that for $t \in \cT$, $s \in [t - \tau, t]$ and any item $a\in\cD$ we have $\alpha_s(a) = \alpha_t(a)$.  
In other words, $\cT$ is the set of time steps that do not follow too close after breakpoints. 
Obviously, for any time step $t \notin \cT$ the estimators of attraction probabilities are biased and $\cascadeswucb$ may recommend suboptimal items constantly. Thus, we get the following bound:
\begin{equation}
\label{eq:swucbdecompose}
\expect{\sum_{t=1}^{n} \mathds{1}\{M_{t}\} \vr_t}  \leq 
L\Upsilon_n \tau + \expect{\sum_{t\in\cT} \mathds{1}\{M_{t}\} \vr_t}  .
\end{equation}
$\tau$ time steps after a breakpoint, the estimators of the attraction probabilities are not biased.

Then, we consider the variance. 
By using a Hoeffding-type inequality (\cite[Corollary 21]{garivier-2008-nonstationary}), for an item $a\in\cD$, $t\in\cT$, and any $\eta > 0$, we get:
\begin{eqnarray*}
	P\left( \abs{\barvalphat(\tau, a) -\alpha_t(a)} > c_t(\tau, t)\right) 
	&\leq&
	P\left( \barvalphat(\tau, a) > \alpha_t(a) +\sqrt{\frac{\epsilon \ln{(t\land \tau)}}{\vN_t(\tau, a)}}\right) 
	+ 	P\left( \barvalphat(\tau, a) < \alpha_t(a) -\sqrt{\frac{\epsilon \ln{(t\land \tau)}}{\vN_t(\tau, a)}}\right)
	\\
	&\leq&
	2\left\lceil \frac{\ln{(t\land \tau)}}{\ln(1+\eta)} \right\rceil \exp\left(-2\epsilon \ln(t\land \tau)(1-\frac{\eta}{16}) \right)
	\\
	&=& 
	2\left\lceil \frac{\ln{(t\land \tau)}}{\ln(1+\eta)} \right\rceil
	(t \land \tau)	^{-2\epsilon(1-\eta^2/16)} .
\end{eqnarray*}
Taking $\eta = 4\sqrt{1-\frac{1}{2\epsilon}}$, we have: 
$P\left( \abs{\barvalphat(\tau, a) -\alpha_t(a)}\right) \leq 2 \frac{\left\lceil \frac{\ln{(t\land \tau)}}{\ln(1+\eta)} \right\rceil}{t\land\tau} .
$
Thus, we get the following bound: 
\[
\expect{\sum_{t\in\cT} \mathds{1}\{M_{t}\} \vr_t} 
\leq 
2L\sum_{t\in\cT} 
\frac{\left\lceil \frac{\ln{(t\land \tau)}}{\ln(1+\eta)} \right\rceil}{t\land\tau}
 \leq
 \frac{L\ln^2(\tau)}{\ln(1+4\sqrt{1-1/2\epsilon)}} + \frac{2Ln\ln \tau}{\tau \ln(1+4\sqrt{1-1/2\epsilon})}.
\]
We sum up and get the upper bound: 
\begin{equation}
\label{eq:swucb bound for outside}
\expect{\sum_{t=1}^{n} \mathds{1}\{\bar{M}_{t}\} \vR_t} 
\leq 
L\Upsilon_n \tau 
+
\frac{L\ln^2(\tau)}{\ln(1+4\sqrt{1-1/2\epsilon)}} 
+ 
\frac{2Ln\ln \tau}{\tau \ln (1+4\sqrt{1-1/2\epsilon})} .
\end{equation}

%%%%%%%%%%%%%%%%%%%%%%%%%%%%%%%
%  Third step
%%%%%%%%%%%%%%%%%%%%%%%%%%%%%%%
Third, we upper bound the second term in Eq.~\ref{eq:swucbR}. 
The regret is caused by recommending a suboptimal item to the user and the user examines but does not click the item.
Since there are $\Upsilon_n$ breakpoints, we refer to $[t_1, \ldots, t_{\Upsilon_n}]$ as the time steps of a breakpoint. 
We consider the time step in the individual epoch that does not contain a breakpoint. 
For any epoch and any time $t \in \{t_e, t_e+1, \ldots, t_{e+1}-1\}$, any item $a \in \cDci$ and any item $a^* \in \cDsi$, we define the event that item $a$ is included in $\vR_t$ instead of item $a^*$ and  item $a$ is examined but not clicked at time $t$ by: 
\begin{equation*}
\begin{split}
G_{t, a, a^*} =  \{ \exists 1\leq k <\vc_t ~s.t.~ \vR_t(k) = a   , \cR_t(k)  = a^*\}.
\end{split}
\end{equation*} 
Since the attraction probabilities remain constant in the epoch, we refer to $\cDsi$ as the optimal items and $\cDci$ as the suboptimal items in epoch $e$. 
By \cite[Theorem 1]{kveton15cascading}, the regret at time $t$ is decomposed as:
\begin{equation}
\label{eq:swucbrdecompose}
\mathbb{E}[\vr_t] \leq  \Delta^t_{a, a^*}\sum_{a \in \cDci} \sum_{a* \in \cDsi} \mathds{1}\{ G_{a, a^*, t} \} .
\end{equation}
Then we have:
\begin{equation}
\label{eq:swucbinside1}
\expect{\sum_{t=t_i}^{t_{i+1}-1} \mathds{1}\{\bar{M}_{t}\} \vr_t} 
= 
\sum_{t=t_i}^{t_{i+1}-1} \mathds{1}\{\bar{M}_{t}\} \expect{\vr_t}
\leq 
\sum_{a \in \cDci} \expect{\sum_{a^*\in\cDsi} \sum_{t=t_i}^{t_{i+1}-1}  
	\Delta^t_{a, a^*} \mathds{1}\{ G_{a, a^*, t} \}} ,
\end{equation}
where the first equality is due to the tower rule, and the inequality if due to Eq.~\ref{eq:swucbrdecompose}. 

Now, for any suboptimal item $a$ in epoch $e$, we upper bound $ \expect{\sum_{a^*\in\cDsi} \sum_{t=t_i}^{t_{i+1}-1}  
	\Delta^t_{a, a^*} \mathds{1}\{ G_{a, a^*, t} \}} $. 
At time $t$,  event $\mathds{1}\{\bar{M}_{t}\}$ and event $ \mathds{1}\{a \in \vR_t, a \in \cDct\}$ happen when there exists an optimal item $a^*\in \cDsi$ such that:
\[
\alpha_t(a) + 2 c_t(\tau, a) \geq \vU_t(a) \geq \vU(a^*) \geq \alpha_t(a^*) , 
\]
which implies that $2 c_t(\tau, a) \geq \alpha_t(a^*) - \alpha_t(a)$. 
Taking the definition of the confidence interval, we get: \[\vN_t(\tau, a) \leq \frac{4\epsilon \ln N_{t}(\tau)}{\Delta^{2}_{t, a, a^*}},\] 
where we set $\Delta_{t, a, a^*} = \Delta^{t}_{ a, a^*}$.

Together with Eq.~\ref{eq:swucbinside1}, we get: 
\begin{eqnarray}
\expect{
	\sum_{t=t_i}^{t_{i+1}-1} \mathds{1}\{\bar{M}_{t}\} \vr_t
} 
&\leq& 
\sum_{a \in \cDci} \expect{\sum_{a^*\in\cDsi} 
	\frac{4\epsilon \ln N_{t}(\gamma)}{\Delta_{t, a, a^*}}}\notag
\\
&\leq&
4\epsilon \ln N_{t}(\gamma)
\left[
\Delta_{t, a, 1}\frac{1}{\Delta^2_{t, a, 1}} + \sum_{a^*=2}^{K} \Delta_{t, a, a^*} \left( \frac{1}{\Delta^2_{t, a, a^*}} - \frac{1}{\Delta^2_{t, a, a^*-1}} \right)
\right] \notag
\\
&\leq& 
\frac{8\epsilon \ln N_{t}(\gamma)}{\Delta_{t, a, K}} ,
\end{eqnarray}
where the last inequality is due to \cite[Lemma 3]{kveton2014matroid}. 
Let $\Delta_{a, K} = \min_{t \in [n]} \Delta_{t, a, K}$ be the smallest gap between the suboptimal item $a$ and an optimal item in all time steps.  
When $\vN_t(\tau, a) > \frac{8\epsilon \ln N_t(\tau)}{\Delta^2_{a, K}}$, $\cascadeducb$ will not select item $a$ at time $t$.
Thus we get: 
\begin{eqnarray}
\label{eq:swucbinside2}
\sum_{a \in \cD} \expect{ \sum_{t=1}^{n} \mathds{1}\{\bar{M}_{t}\} \mathds{1}\{a \in \vR_t, a \in \cDct\}} &\leq& 
\sum_{e\in [\Upsilon_n]} \sum_{a \in \cDci} \frac{8\epsilon \ln N_{t}(\tau)}{\Delta_{t, a, K}}   \notag
\\
&\leq&
\sum_{a \in \cD} \left\lceil \frac{n}{\tau} \right\rceil 
\frac{8\epsilon \ln (n\land\tau)}{\Delta_{a, K}}  ,
\end{eqnarray}
where the last inequality is based on \cite[Lemma 25]{garivier-2008-nonstationary}.

Finally, together with Eq.~\ref{eq:swucbR}, Eq.~\ref{eq:swucbdecompose},
Eq.~\ref{eq:swucb bound for outside}, 
Eq.~\ref{eq:swucbinside1} and Eq.~\ref{eq:swucbinside2}, we get \cref{th:upperboundcascadeswucb}. 

\section{Further Experiments}
In this section, we compare  $\cascadeducb$, $\cascadeswucb$ and baselines on single queries. 
We pick 20 queries and report the results in~\cref{fig:singleregret}. 
The results exemplify that $\cascadeducb$ and $\cascadeswucb$ have sub-linear regret while other baselines have linear regret. 
\begin{figure}[t]
	\includegraphics{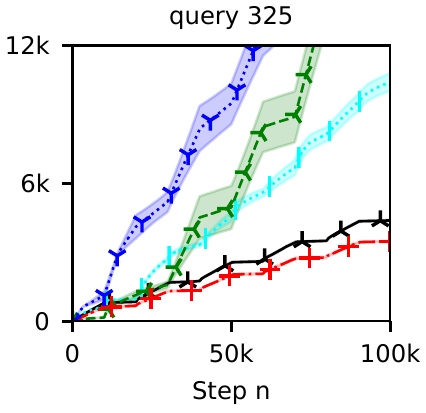}
	\includegraphics{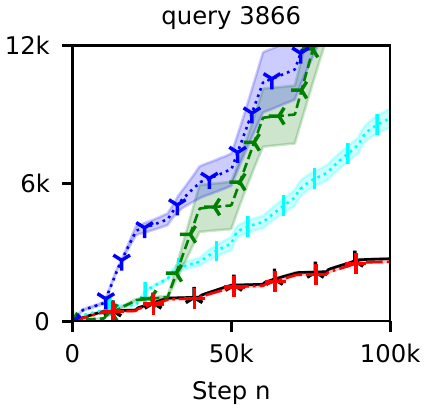}
	\includegraphics{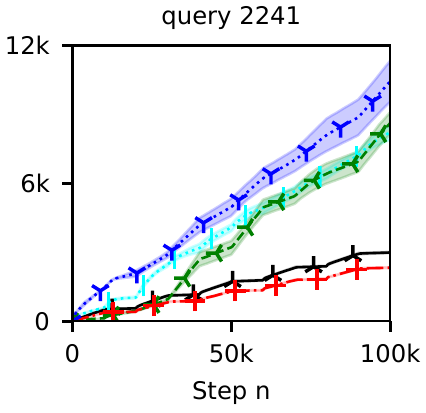}
	\includegraphics{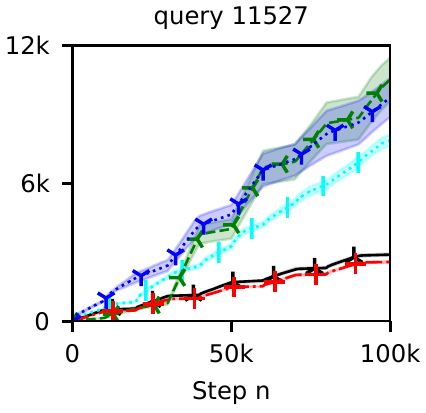}
	\includegraphics{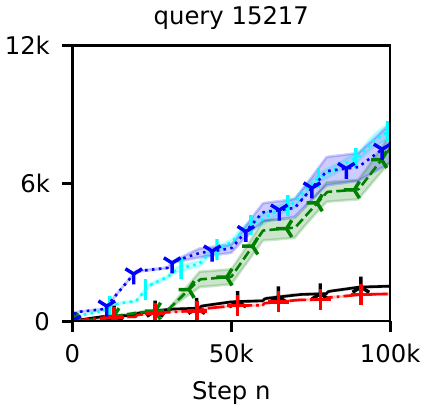}
	\includegraphics{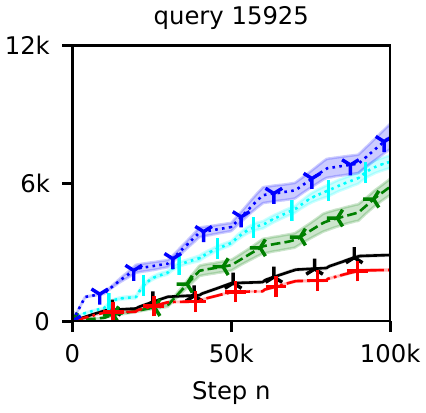}
	\includegraphics{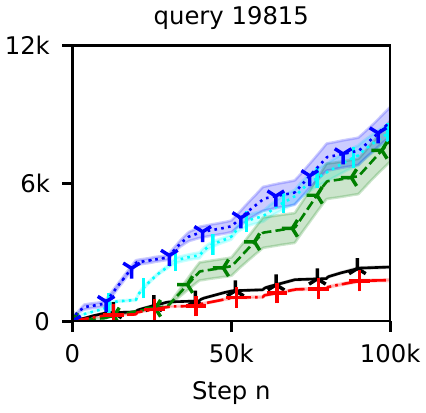}
	\includegraphics{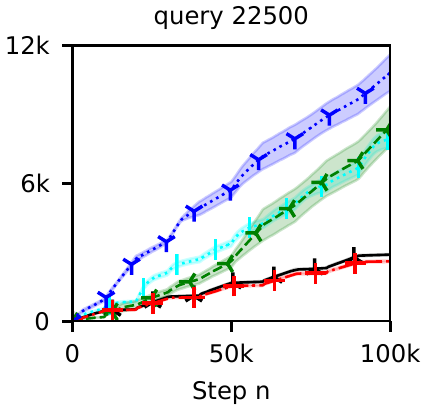}
	\includegraphics{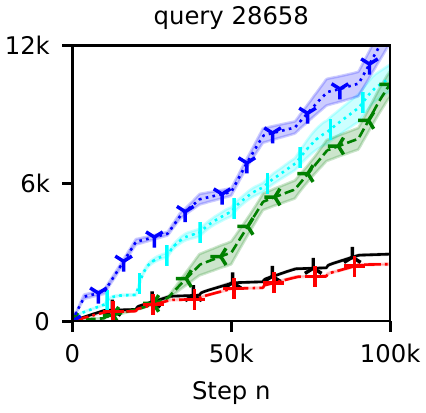}
	\includegraphics{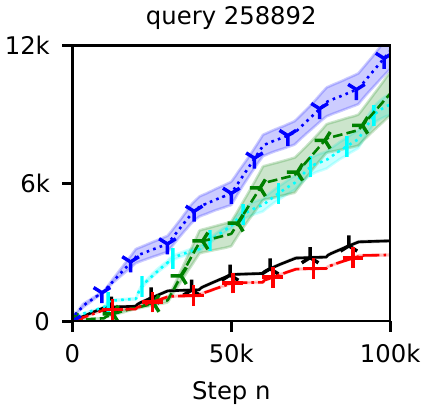}
	\includegraphics{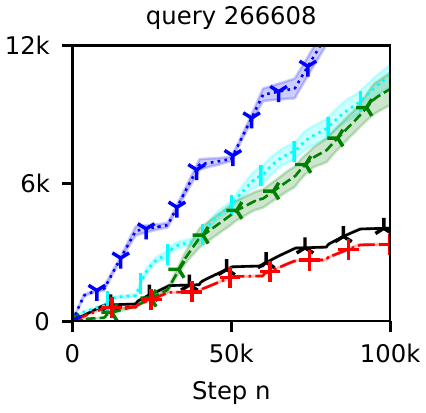}
	\includegraphics{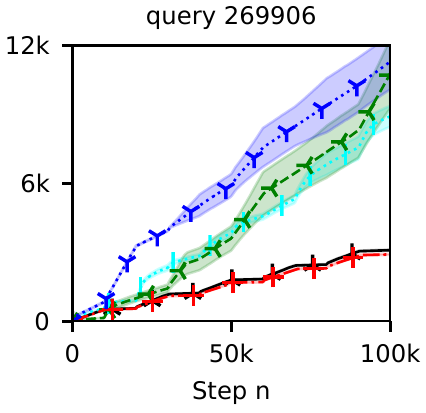}
	\includegraphics{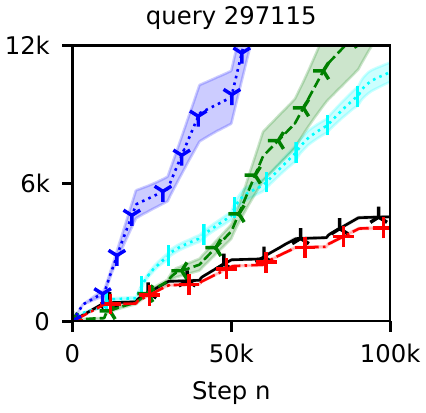}
	\includegraphics{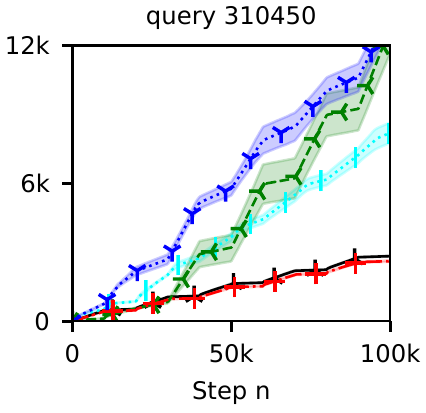}
	\includegraphics{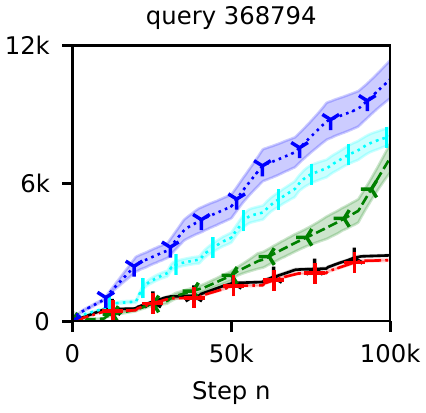}
	\includegraphics{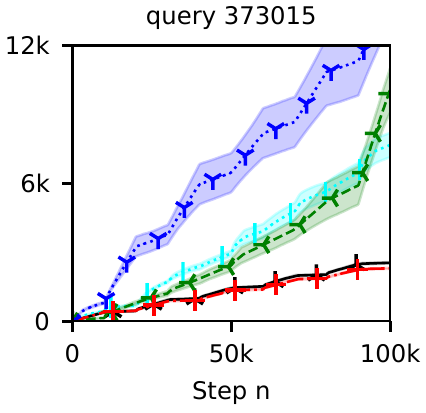}
	\includegraphics{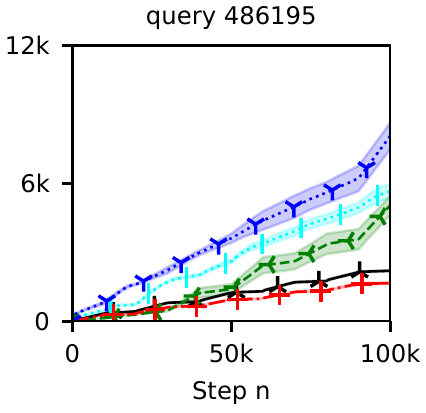}
	\includegraphics{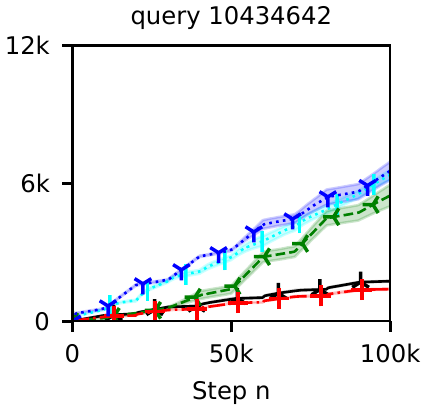}
	\includegraphics{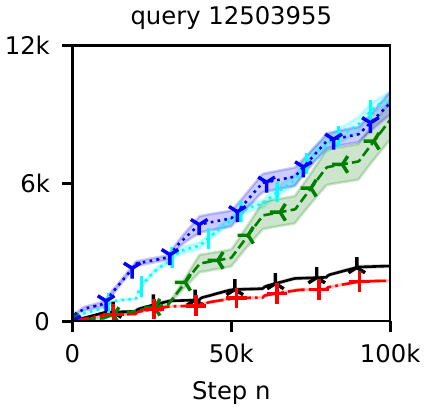}
	\includegraphics{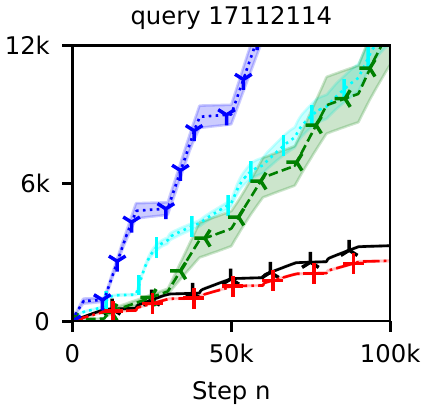}
	\caption{The n-step regret of $\cascadeducb$ (black), $\cascadeswucb$ (red), $\rankedexpthree$ (cyan), $\cascadeklucb$ (green) and $\bubblerank$ (blue)  on single queries in up to $100$k steps. Lower is better. The results are averaged over $10$ runs per query. The shaded regions represent standard errors of our estimates.}\label{fig:singleregret}
\end{figure}

\end{document}